\pdfoutput=1

\documentclass[11pt]{article}

\usepackage[final]{acl}

\usepackage{times}
\usepackage{latexsym}

\usepackage[T1]{fontenc}

\usepackage[utf8]{inputenc}

\usepackage{microtype}

\usepackage{inconsolata}

\usepackage{graphicx}
\usepackage{amsmath}
\usepackage{amssymb}
\usepackage{mathtools}
\usepackage{amsthm}
\usepackage{hyperref}
\usepackage{amssymb}
\usepackage{multirow}
\usepackage{xcolor}
\usepackage{colortbl}
\usepackage{booktabs}
\usepackage{CJKutf8}

\usepackage{algorithm}
\usepackage{algorithmicx}
\usepackage{algpseudocode}
\usepackage{verbatim}

\theoremstyle{plain}
\newtheorem{theorem}{Theorem}[section]

\theoremstyle{definition}

\theoremstyle{remark}

\definecolor{marron}{HTML}{CC88B0}

\title{TL-Training: A Task-Feature-Based Framework for Training\\Large Language Models in Tool Use}

\author{
\bf{Junjie Ye$^{1}$, Yilong Wu$^{1}$, Sixian Li$^{1}$, Yuming Yang$^{1}$,}\\
\bf{Zhiheng Xi$^{1}$, Tao Gui$^{1,4}$, Qi Zhang$^{1,4,5}$\thanks{Corresponding Author.}, Xuanjing Huang$^{1,4,5}$,}\\
\bf{Peng Wang$^{2}$, Zhongchao Shi$^{2}$, Jianping Fan$^{2}$, Zhengyin Du$^3$}\vspace{2mm}\\
{$^1$Fudan University}\ \ 
{$^2$Lenovo Research, Beijing, China}\ \ 
{$^3$ByteDance Seed}\\
{$^{4}$Shanghai Key Lab of Intelligent Information Processing}\ \ 
{$^{5}$Shanghai AI Laboratory}\vspace{2mm}\\
\texttt{jjye23@m.fudan.edu.cn, qz@fudan.edu.cn} 
}

\begin{document}
\maketitle
\begin{abstract}
Large language models (LLMs) achieve remarkable advancements by leveraging tools to interact with environments, a critical step toward generalized AI. However, the standard supervised fine-tuning (SFT) approach, which relies on large-scale datasets, often overlooks task-specific characteristics in tool use, leading to performance bottlenecks.
To address this issue, we analyze three existing LLMs and uncover key insights: training data can inadvertently impede tool-use behavior, token importance is distributed unevenly, and errors in tool calls fall into a small set of categories. Building on these findings, we propose~\emph{TL-Training}, a task-feature-based framework that mitigates the effects of suboptimal training data, dynamically adjusts token weights to prioritize key tokens during SFT, and incorporates a robust reward mechanism tailored to error categories, optimized through proximal policy optimization.
We validate TL-Training by training CodeLLaMA-2-7B and evaluating it on four open-source test sets. Our results demonstrate that the LLM trained by our method matches or surpasses both open- and closed-source LLMs in tool-use performance using only 1,217 training data points. Additionally, our method enhances robustness in noisy environments and improves general task performance, offering a scalable and efficient paradigm for tool-use training in LLMs. Code and data are available at~\url{https://github.com/Junjie-Ye/TL-Training}.
\end{abstract}

\section{Introduction}

\begin{figure}[!t]
    \centering
    \includegraphics[width=\linewidth]{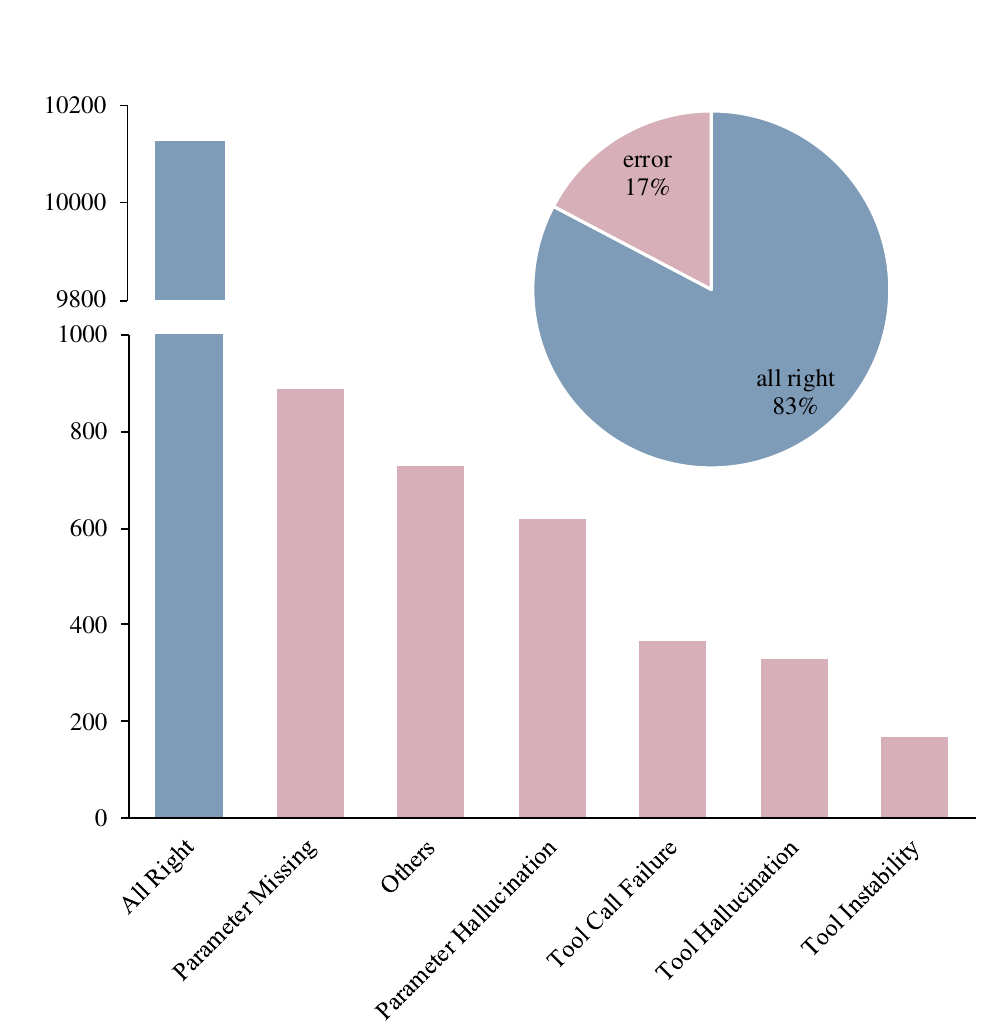}
    \caption{Error statistics for various tool calls in RoTLLaMA's training data.}
    \label{fig:error}

\end{figure}

Large language models (LLMs)~\cite{GPT-4, LLaMA-2, Qwen} excel in natural language understanding due to pre-training on extensive datasets~\cite{analysis-chen, analy-ye}. By incorporating tool-use capabilities, LLMs can extend beyond text generation to interact with the external environment, enabling tasks such as web searches and email management~\cite{Toolalpaca, ToolEyes}. Furthermore, these capabilities are essential for addressing real-world user needs and advancing the development of general-purpose AI~\cite{agent-survey}.  

Current approaches to training LLMs for tool use rely heavily on large-scale datasets generated from trajectories of interactions with tools~\cite{Toolllm, ToolChain}. Standard supervised fine-tuning (SFT) is then applied to pre-trained models. While effective in some cases, these methods overlook key task-specific characteristics, leading to performance bottlenecks. For instance, ToolLLaMA-2-7B-v2~\cite{Toolllm} achieves only 80\% of GPT-4’s performance on tool-use benchmarks~\cite{ToolEyes, Sealtools}, indicating room for significant improvement.  

To fill this gap, we conduct an in-depth analysis of three tool-using LLMs, uncovering several key phenomena. Notably, over 17\% of the training data for RoTLLaMA~\cite{RoTBench} contains tool-calling errors (Figure~\ref{fig:error}), primarily due to reliance on data from GPT series models, which are not entirely error-free with complex tools. Training on such flawed data can hinder model performance. Additionally, our tests of ToolLLaMA-2-7B-v2 and NexusRaven-13B-v2~\cite{nexusraven} reveal that incorrect tool selections often share a common prefix with the correct ones (Table~\ref{tab:error}), and correcting initial erroneous tokens can lead to successful predictions. This suggests that certain tokens are more critical in tool selection. Moreover, the types of errors produced by tool calls are relatively limited (Figure~\ref{fig:information}), providing a foundation for targeted improvements across various error categories.

\begin{table}[!t]
    \caption{Proportion statistics of various error cases in RoTBench (Clean) for ToolLLaMA-2-7B-v2 and NexusRaven-13B-v2. `First' indicates a mismatch in the initial token of the selected and correct tool names, while `Prefix' denotes a shared prefix between them. `Synonyms' captures instances where filled parameter values are synonymous with standard values.}
    \label{tab:error}
    \centering
    \resizebox{\linewidth}{!}
    {
    \begin{tabular}{llcc}
    \toprule
     \textbf{Aspect} & \textbf{Error} & \textbf{ToolLLaMA} & \textbf{NexusRaven} \\
     \midrule
      \multirow{2}*{Tool}  & First & 53.33\% & 65.22\% \\
      & Prefix & 46.67\% & 34.78\% \\
      \midrule
      \multirow{3}*{Parameter} & Redundancy & 46.43\% & 33.33\% \\
      & Missing & 57.14\% & 66.67\% \\
      & Hallucination & 17.86\% & 6.67\% \\
      \midrule
      \multirow{2}*{Content} & Synonyms & 73.68\% & 95.00\% \\
      & Others & 31.58\% & 5.00\% \\
    \bottomrule
    \end{tabular}
    }    

\end{table}

Based on these insights, we propose~\emph{TL-Training}, a task-feature-based framework for training LLMs in tool use. TL-Training mitigates the negative impact of training data by identifying erroneous interaction paths and excluding them from gradient updates. It prioritizes key tokens through adaptive weighting during SFT and incorporates tool feedback into a robust reward mechanism for reinforcement learning using the proximal policy optimization (PPO)~\cite{PPO}.  

We validate our approach by training CodeLLaMA-2-7B~\cite{CodeLLaMA} on a curated dataset of 1,217 tool-call trajectories generated with GPT-4o. Evaluations on four open-source test sets demonstrate that the model trained with TL-Training matches or surpasses the tool-use performance of leading open- and closed-source LLMs, despite requiring significantly less training data. Additionally, TL-Training improves robustness in noisy environments and enhances general task performance.  

In summary, our contributions are as follows:
\begin{itemize}
\item We identify three key insights in tool use, including the impact of erroneous data, the uneven importance of tokens, and the constrained range of tool-calling error categories.
\item We propose TL-Training, a novel task-feature-based framework  comprising of adverse effects mitigation, key tokens prioritization, and reinforcement learning for tool use.
\item We demonstrate the effectiveness of TL-Training by training CodeLLaMA-2-7B and achieving leading tool-use performance on multiple benchmarks with only 1,217 pieces of data.
\item We show that TL-Training enhances both robustness to noisy data and general task performance, highlighting its potential for scalable tool-use training.
\end{itemize}

\begin{figure*}[!t]
    \centering
    \includegraphics[width=\linewidth]{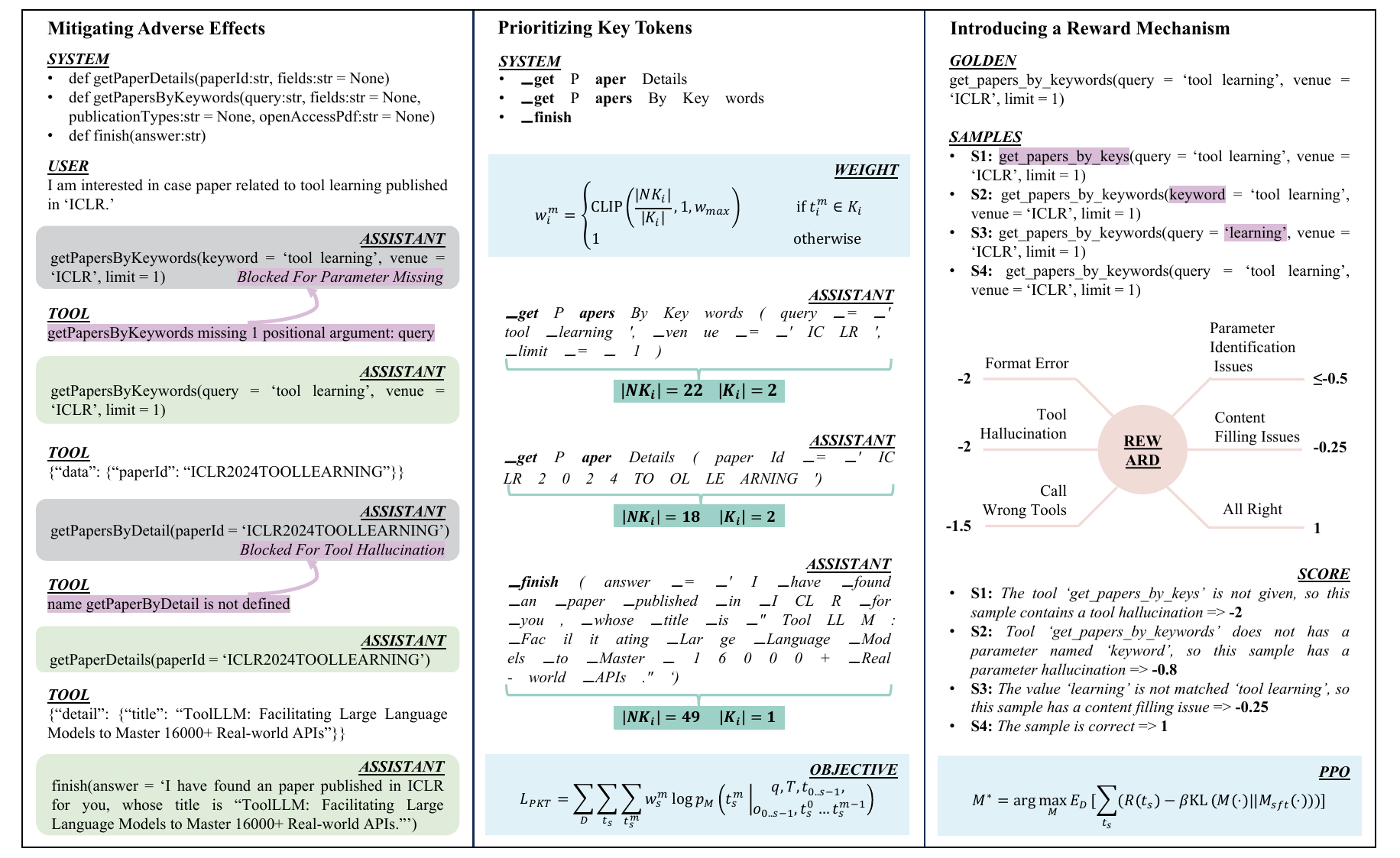}
    \caption{Framework of~\emph{TL-Training}. TL-Training comprises three main components: (\textbf{Left}) mitigating the adverse effects of suboptimal data by identifying erroneous interaction trajectories through tool feedback and blocking their gradient updates; (\textbf{Middle}) optimizing key tokens by dynamically adjusting token weights during the SFT process; and (\textbf{Right}) enhancing tool call performance through a reward mechanism tailored to tool invocation error types, using the PPO algorithm for reinforcement learning.}
    \label{fig:approach}

\end{figure*}

\section{Related Works}

\paragraph{Training LLMs for Tool Use}  
LLMs capable of utilizing external tools significantly enhance their ability to interact with dynamic environments and address user needs~\cite{tool-learning}. However, the diversity and complexity of real-world tools present significant challenges in training such models. Existing methods, such as SFT, rely on the generation of extensive datasets of tool interactions~\cite{RestGPT, Toolalpaca, ToolVerifier}, enabling models to learn tool functionalities, invoke appropriate tools, and process feedback. While effective, these methods are resource-intensive due to the large-scale data construction involved.  
To overcome these challenges, some studies have proposed encoding tool names as special tokens directly integrated into model training, embedding tool-specific knowledge into the model~\cite{Toolken}. This approach has shown promise for existing tools but remains limited in its ability to adapt to newly introduced tools. Building on these findings, our work introduces a novel training paradigm for tool-use LLMs, addressing both efficiency and adaptability. By leveraging a compact dataset of 1,217 data points and incorporating three task-specific components, our approach achieves state-of-the-art performance while significantly reducing data requirements.

\paragraph{Evaluating LLMs in Tool Use}  
Evaluating LLMs' tool-use capabilities is essential for understanding their effectiveness in diverse scenarios. A common evaluation method involves comparing predicted outputs with standard answers from a single turn of tool use~\cite{TEval}. However, in multi-turn interactions, the variability in invocation processes complicates the definition of a single standard path. To address this, evaluations increasingly consider multiple dimensions of tool-use processes and outcomes~\cite{ToolEyes}.  
Beyond tool-use performance, researchers have also investigated robustness and safety in practical scenarios~\cite{RoTBench, ToolSword}, which provide insights into how LLMs manage edge cases and avoid harmful outputs. In this paper, we evaluate LLMs across single-turn and multi-turn tool use to provide a more comprehensive assessment. Additionally, we analyze robustness to further demonstrate the superiority of our approach.

\section{Preliminaries}
\label{sec:pre}

\paragraph{Task Formulation}  
Given a model \( \mathcal{M} \), a user query \( q \), and a collection of tools \( \mathbb{T} \), the task of tool use requires \( \mathcal{M} \) to iteratively select the appropriate tool \( t_s \in \mathbb{T} \) at each step \( s \), process its feedback \( o_s \), and continue selecting subsequent tools \( t_{s+1} \) until the query is resolved and a final answer is obtained. Formally, this can be represented as \( t_{s+1} = \mathcal{M}(\cdot | q, \mathbb{T}, t_{0..s}, o_{0..s}) \).
This task is distinct from traditional natural language processing tasks, as it requires the model to invoke tools repeatedly and interpret their feedback dynamically. Despite its importance, to the best of our knowledge, there has been no systematic examination of the intrinsic properties of tool use. Thus, we aims to fill this gap by conducting an in-depth analysis focusing on both the training data and model performance.

\paragraph{Data Analysis}  
For our analysis, we use the training set from RoTLLaMA, which includes 12,247 filtered multi-turn tool-call trajectories generated by GPT-4. Illustrated in Figure~\ref{fig:error}, 17\% of these trajectories contain various errors, indicating that even advanced models like GPT-4 encounter challenges with complex tools. These erroneous trajectories pose a challenge for models trained through SFT, as they inherit these error patterns during learning. This predisposition to incorrect tool invocation highlights the need for more robust training methods to mitigate error propagation and improve overall model performance.

\paragraph{Performance Analysis}  
We evaluate the performance of ToolLLaMA-2-7B-v2 and NexusRaven-13B-v2, which are built on LLaMA-2-7B~\cite{LLaMA-2} and CodeLLaMA-2-13B~\cite{CodeLLaMA}, respectively, as representative tool-using LLMs. Our evaluation uses RoTBench (Clean)~\cite{RoTBench}, a manually labeled dataset for the single-turn tool-use task with standardized answers. The analysis focuses on errors related to tool selection, parameter identification, and content filling, with results shown in Table~\ref{tab:error}. We observe that when the model selects the wrong tool, it often chooses one with a prefix similar to the correct tool. By manually correcting the first incorrectly predicted token, the model can generate the correct one, suggesting that certain tokens are crucial for task success. Additionally, errors in parameter identification and content generation highlight areas where further training is required.


\section{Approaches}

Building on the analysis in Section~\ref{sec:pre}, we propose TL-Training, a novel training paradigm for LLMs in tool use. As shown in Figure~\ref{fig:approach}, this paradigm incorporates three core techniques: mitigating the adverse effects of suboptimal data by preventing its back-propagation (Section~\ref{sec:block}), prioritizing key tokens using adaptive weight adjustments (Section~\ref{sec:adapt}), and implementing a reward mechanism tailored to tool invocation error categories to enable effective reinforcement learning (Section~\ref{sec:rl}).\footnote{Theoretical proofs of the effectiveness of our proposed approaches is provided in Appendix~\ref{sec:proof}.}

\subsection{Mitigating Adverse Effects}
\label{sec:block}

During the SFT stage, the objective is to align LLMs with the distribution of the training data. However, erroneous interaction paths in the data can negatively affect the model's decision-making, leading to an increased likelihood of incorrect tool calls. To address this, we design an automated process that identifies erroneous interaction paths and blocks their back-propagation, thereby reducing their harmful impact on the model.

Given a data sequence \( (q, t_{0..s}, o_{0..s})\), we seek to identify the erroneous tool call trajectory \(\mathbb{T}_e \subseteq \{t_0, t_1, \dots, t_s\}\). Directly determining whether a specific \(t_i\) is correct is challenging. However, the feedback \(o_i\) generated after each tool call contains structured error-reporting information, as summarized in Figure~\ref{fig:information}.\footnote{Specific examples can be found in Appendix~\ref{sec:information}.} We automate the identification of incorrect calls by sequentially analyzing \(o_i\) to extract \(\mathbb{T}_e\).

Once \(\mathbb{T}_e\) is identified, we mitigate the impact of these erroneous interactions by blocking their back-propagation during training. This is achieved by modifying the loss function as follows:
\[
\begin{split}
\mathcal{L}_{MAE} = & -\sum_{\mathbb{D}} \sum_{t_{s} \notin \mathbb{T}_e} \log p_M(t_{s} | q, \mathbb{T},\\
& t_{0..s-1}, o_{0..s-1}),
\end{split}
\]
where \(\mathbb{D}\) represents the entire training dataset.

\begin{figure}
    \centering
    \includegraphics[width=0.8\linewidth]{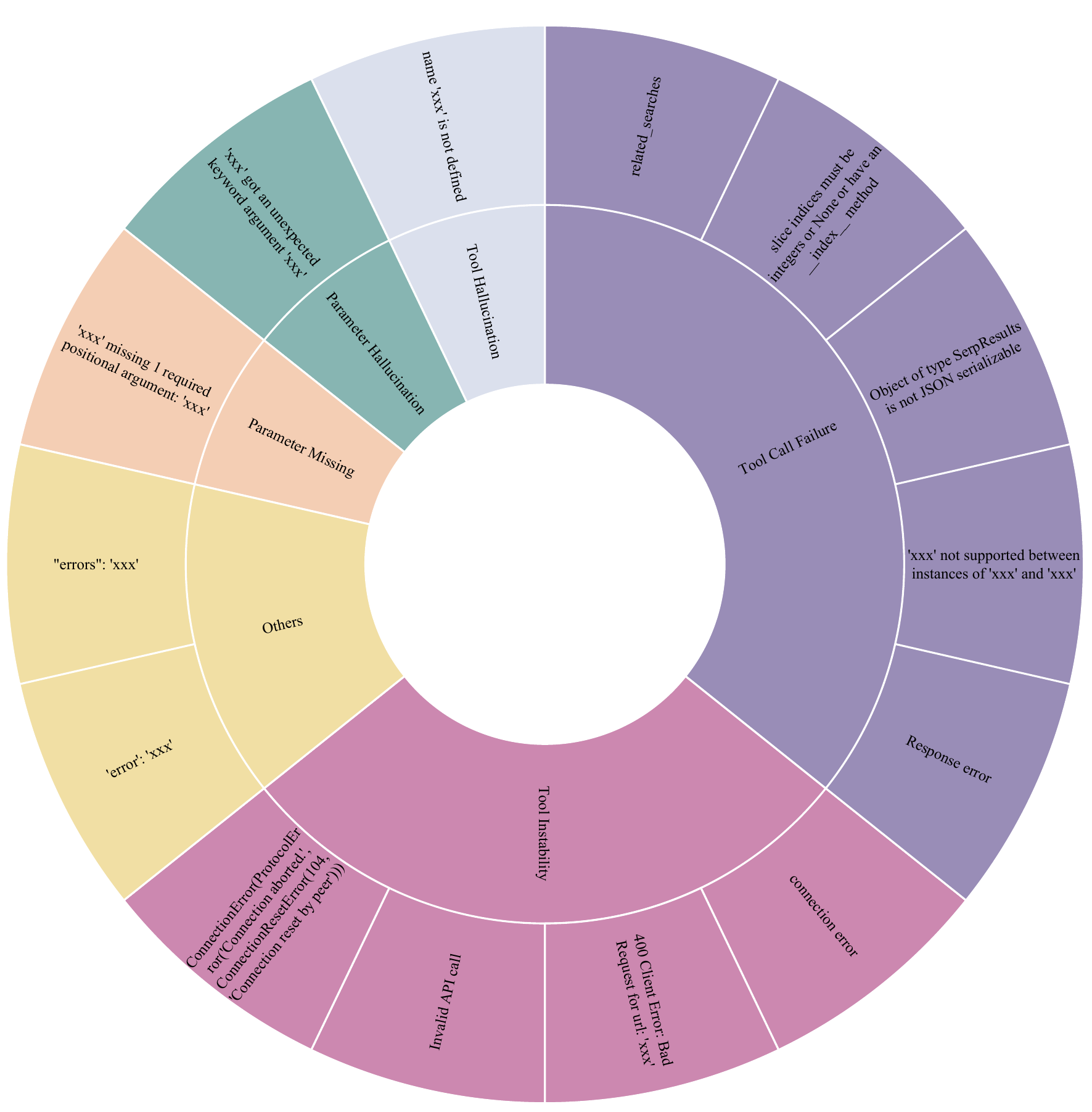}
    \caption{Types of errors encountered by LLMs during tool use and their corresponding feedback messages.}
    \label{fig:information}

\end{figure}

\subsection{Prioritizing Key Tokens}
\label{sec:adapt}
Based on the analysis in Section~\ref{sec:pre}, and the insights from rows 1-2 of Table~\ref{tab:error}, we observe that the first token of a tool name, along with any subsequent token that shares a common prefix with other tool names, plays a more critical role in successful tool identification. As such, these tokens are more challenging for LLMs to generate correctly. However, standard SFT training maximizes the conditional probability of each token without distinction, treating all tokens as equally important. To address this limitation, we propose a scheme that adaptively adjusts the training weights of tokens according to their relative importance.

Given a data sequence \((q, t_{0..s}, o_{0..s})\), where each tool \(t_i = (t_i^0, t_i^1, \ldots, t_i^{l_i})\) consists of \(l_i\) tokens, we categorize the tokens into two sets: 
\[
K_i = \{t_i^m \in t_i \mid t_i^m \text{ is a key token}\}
\]
\[
NK_i = \{t_i^m \in t_i \mid t_i^m \text{ is not a key token}\}
\]
We then adjust the weights of \(K_i\) and \(NK_i\) based on their relative importance, allowing the model to focus more on the key tokens.

\[
w_i^m = 
\begin{cases} 
\text{CLIP}\left(\frac{|NK_i|}{|K_i|}, 1, w_{\text{max}}\right) & \text{if } t_i^m \in K_i \\ 
1 & \text{otherwise}
\end{cases}
\]

Here, \(w_{\text{max}}\) is the maximum adjustment multiplier, and \(\text{CLIP}(x, \text{min}, \text{max})\) is used to constrain the adjustment factor to lie within the range \([\text{min}, \text{max}]\). The notation \(|\cdot|\) represents the size of the set.\footnote{Since \(K_i\) always includes at least the first token of the tool name, we avoid any risk of dividing by zero.}

With these computed weights, we prioritize key tokens during training with the following objective:
\[
\begin{split}
\mathcal{L}_{PKT} = & -\sum_\mathbb{D} \sum_{t_s} \sum_{t_s^m} w_s^m \cdot \log p_M(t_s^m \mid q, \mathbb{T},\\
& t_{0..s-1}, o_{0..s-1}, t_s^0 \ldots t_s^{m-1}).
\end{split}
\]

\subsection{Introducing a Reward Mechanism}
\label{sec:rl}
The three stages of tool use by LLMs are interdependent, where an error in any stage can lead to the failure of the entire tool invocation. Fortunately, the types of errors that arise are limited, enabling us to introduce a reward mechanism based on these specific errors. This allows us to apply reinforcement learning algorithms that help align the model more closely with human intent and enhance its tool-use proficiency. To achieve this, we define a set of reward functions tailored to the tool use task and employ the PPO algorithm to optimize the model’s performance.

Given an LLM-generated tool call prediction \( t_i \) and its corresponding ground truth, we define the following reward function based on the quality of the LLM's tool use in various scenarios:
\[
R(t_i) =
\begin{cases} 
  -2 & \text{if } t_i \text{ cannot be parsed} \\
  -2 & \text{if } t_i \text{ contains tool hallucinations} \\
  -1.5 & \text{if } t_i \text{ calls the wrong tool} \\
  R_{\text{p}}(t_i) & \text{if } t_i \text{ has parameter issues} \\
  -0.25 & \text{if } t_i \text{ has content filling issues} \\
  1 & \text{if } t_i \text{ is correct}
\end{cases}
\]
where \( R_{\text{p}}(t_i) \) is defined as:
\[
\begin{split}
    R_{\text{p}}(t_i) = & -0.8 \cdot \mathbb{I}(t_i \text{ has parameter hallucinations}) \\
     & - 0.5 \cdot \mathbb{I}(t_i \text{ has redundant parameters})\\
     & - 0.5 \cdot \mathbb{I}(t_i \text{ has missing parameters})
\end{split}
\]
where \(\mathbb{I}(\cdot)\) represents the indicator function.

This reward function \( R \) addresses the different potential errors in LLM tool use, providing a structured scoring system to assess performance. Based on this, we apply the PPO algorithm, which iteratively optimizes the model’s parameters to maximize these rewards as follows:
\[
\begin{split}
    \mathcal{M}^* = & \arg\max_\mathcal{M} \mathbb{E}_\mathbb{D} [\sum_{t_s} (R(t_s)- \\
    & \beta \text{KL}(\mathcal{M}(\cdot)||\mathcal{M}_{sft}(\cdot)))]
\end{split}
\]

where $\beta$ regulates deviation from the initial SFT model $\mathcal{M}_{sft}$.
This approach enables the LLM to progressively refine its understanding and improve the accuracy of its tool usage over time.

\begin{table}[!t]
\caption{Statics of datasets used. `\# Number' represents the number of data in the dataset. `Type' represents the type of tool use.}
\label{tab:dataset}
\centering
\resizebox{\linewidth}{!}
{
\begin{tabular}{l lcc} 
\toprule
 \textbf{Split} &  \textbf{Dataset} &  \textbf{\# Number} & \textbf{Type} \\
 \midrule
 Train & Self-Construct& 1217 & Multi-Turn   \\
 \midrule
 \multirow{4}*{Test} & ToolAlpaca & 114 & Single-Turn  \\
 & RoTBench & 105 & Single-Turn  \\
 & BFCL-v3 & 239 & Single-Turn   \\
 & ToolEyes& 382 & Multi-Turn   \\
\bottomrule
\end{tabular}
}

\end{table}

\section{Experimental Setup}

\subsection{Dataset}

As shown in Table~\ref{tab:dataset}, to validate our approach, we construct a custom training set focused on multi-turn tool use and evaluate it using four publicly available test sets (i.e., ToolAlpaca~\cite{Toolalpaca}, RoTBench~\cite{RoTBench}, BFCL-v3~\cite{BFCL}, and ToolEyes).\footnote{Details of the datasets can be found in Appendix~\ref{sec:detail-datasets}}

\begin{table*}[!t]
    \caption{Performance of various LLMs on single-turn test sets. `Avg.' represents the average performance across all LLMs. Individual LLM performances are color-coded for clarity: \colorbox{teal}{teal} highlights better-than-average performance, while \colorbox{marron}{purple} indicates below-average performance. Darker shades signify greater deviations from the average. The best performance in each column is indicated in \textbf{bold}.}
    \label{tab:main_single}
    \centering
    \resizebox{\linewidth}{!}
    {
    \begin{tabular}{llccccccccc}
    \toprule
         \multirow{2}*{\textbf{Models}}& \multirow{2}*{\textbf{Size}} & \multicolumn{3}{c}{\textbf{ToolAlpaca}} & \multicolumn{3}{c}{\textbf{RoTBench}} & \multicolumn{3}{c}{\textbf{BFCL-v3}}  \\ \cmidrule(lr){3-5} \cmidrule(lr){6-8} \cmidrule(lr){9-11} 
&  & \textbf{TS} ($\uparrow$) & \textbf{PI} ($\uparrow$) & \textbf{CF} ($\uparrow$) & \textbf{TS} ($\uparrow$) & \textbf{PI} ($\uparrow$) & \textbf{CF} ($\uparrow$) & \textbf{TS} ($\uparrow$) & \textbf{PI} ($\uparrow$) & \textbf{CF} ($\uparrow$) \\ \midrule
    \textit{Avg.} & & \textit{80.63} & \textit{65.09} & \textit{42.72} & \textit{74.10} & \textit{49.90} & \textit{35.43} & \textit{93.13} & \textit{89.88} &\textit{74.17} \\ \midrule
         ToolLLaMA-2& 7B & \cellcolor{marron!15}75.56 & \cellcolor{marron!12}61.40 & \cellcolor{marron!15}37.72 & \cellcolor{marron!12}70.48 & \cellcolor{marron!18}43.81 & \cellcolor{marron!30}25.71 & \cellcolor{marron!15}87.08 & \cellcolor{marron!18}83.75 & \cellcolor{marron!54}56.67 \\
         NexusRaven-2&13B & \cellcolor{teal!6}82.46 & \cellcolor{marron!51}48.25 & \cellcolor{marron!15}37.72 & \cellcolor{marron!12}70.48 & \cellcolor{teal!21}{56.19} & \cellcolor{teal!6}37.14 & \cellcolor{teal!12}97.08 & \cellcolor{teal!15}94.17 & \cellcolor{teal!3}75.83 \\ \midrule
         ChatGLM-4-chat&9B & \cellcolor{marron!21}73.68 & \cellcolor{teal!9}68.42 & \cellcolor{marron!12}38.60 & \cellcolor{marron!21}67.62 & \cellcolor{teal!12}53.33 & \cellcolor{teal!6}37.14 & \cellcolor{teal!6}95.42 & \cellcolor{teal!12}93.33 & \cellcolor{teal!36}86.67 \\
         Qwen-2-Instruct&7B & \cellcolor{teal!18}86.84 & \cellcolor{teal!9}68.42 & \cellcolor{teal!3}43.86 & \cellcolor{teal!1}74.29 & \cellcolor{marron!6}47.62 & \cellcolor{marron!1}35.24 & \cellcolor{teal!15}98.33 & \cellcolor{teal!18}{95.42} & \cellcolor{teal!33}85.00 \\ 
         LLaMA-3.1-Instruct & 8B & \cellcolor{teal!12}84.21 & \cellcolor{marron!18}59.65 & \cellcolor{teal!1}42.98 & \cellcolor{marron!36}62.86 & \cellcolor{marron!96}17.14 & \cellcolor{marron!81}8.57& \cellcolor{marron!90}63.75 & \cellcolor{marron!100}59.58 & \cellcolor{marron!90}34.58 \\
         Qwen-2.5-Instruct& 7B & \cellcolor{teal!36}{\textbf{92.11}} & \cellcolor{teal!9}68.42 & \cellcolor{teal!6}44.74 & \cellcolor{teal!18}80.00 & \cellcolor{marron!51}32.38 & \cellcolor{marron!60}15.24 & \cellcolor{teal!21}{\textbf{100.00}} & \cellcolor{teal!18}{95.42} & \cellcolor{teal!30}87.92\\
         \midrule
         GPT-3.5-turbo & -& \cellcolor{marron!24}72.81 & \cellcolor{marron!33}54.39 & \cellcolor{marron!9}39.47 & \cellcolor{teal!1}74.29 & \cellcolor{teal!36}61.90 & \cellcolor{teal!39}48.57 & \cellcolor{teal!18}99.17 & \cellcolor{teal!18}\textbf{95.83} & \cellcolor{teal!3}75.83 \\
         GPT-4o & -& \cellcolor{marron!12}76.32 & \cellcolor{teal!15}70.18 & \cellcolor{marron!2}42.11 & \cellcolor{teal!1}74.29 & \cellcolor{teal!39}62.86 & \cellcolor{teal!45}50.48 & \cellcolor{teal!9}96.67 & \cellcolor{teal!12}93.75 & \cellcolor{teal!1}74.58 \\
         GPT-4-turbo & -& \cellcolor{marron!18}74.56 & \cellcolor{teal!24}73.68 & \cellcolor{marron!2}42.11 & \cellcolor{teal!24}82.86 & \cellcolor{teal!60}\textbf{69.52} & \cellcolor{teal!54}\textbf{53.33} & \cellcolor{teal!12}97.50 & \cellcolor{teal!12}93.75 & \cellcolor{teal!6}76.25 \\ \midrule
         TL-CodeLLaMA-2&7B & \cellcolor{teal!21}87.72 & \cellcolor{teal!39}{\textbf{78.07}} & \cellcolor{teal!45}{\textbf{57.89}} & \cellcolor{teal!27}{\textbf{83.81}} & \cellcolor{teal!15}{54.29} & \cellcolor{teal!21}{42.86} & \cellcolor{teal!9}96.25 & \cellcolor{teal!12}93.75 & \cellcolor{teal!52}{\textbf{88.33}} \\ 
 \bottomrule
    \end{tabular}
    }
\end{table*}

\subsection{Baselines}
We conduct a comprehensive comparison of ten LLMs from three different categories. These include \textbf{ToolLLaMA-2-7B} and \textbf{NexusRaven-2-13B} as tool-use LLMs; \textbf{ChatGLM-4-chat-9B} \cite{ChatGLM}, \textbf{Qwen-2-Instruct-7B} \cite{Qwen-2}, \textbf{LLaMA-3.1-Instruct-8B} \cite{LLaMA3.1}, and \textbf{Qwen-2.5-Instruct-7B} \cite{Qwen2.5} as open-source LLMs; \textbf{GPT-3.5-turbo}, \textbf{GPT-4o}, and \textbf{GPT-4-turbo} as closed-source LLMs; and \textbf{TL-CodeLLaMA-2}, developed from CodeLLaMA-2-7B with the custom dataset.\footnote{Details of baselines can be found in Appendix~\ref{sec:detail-base}.}

\begin{table}[!t]
    \caption{Performance on the multi-turn test set.
    }
    \label{tab:main_multi}
    \centering
    \resizebox{\linewidth}{!}
    {
    \begin{tabular}{lccc}
    \toprule
         \multirow{2}*{\textbf{Models}} & \multicolumn{3}{c}{\textbf{ToolEyes}} \\ \cmidrule(lr){2-4}
         &  \textbf{DE} ($\downarrow$) & \textbf{CE} ($\downarrow$) & \textbf{VA} ($\uparrow$) \\ \midrule
         \textit{Avg.} &\textit{3.91} &\textit{12.46} & \textit{65.56} \\ \midrule
         ToolLLaMA-2  & \cellcolor{marron!54}21.00 & \cellcolor{marron!72}36.62 & \cellcolor{marron!39}52.36\\ \midrule
         ChatGLM-4-chat & \cellcolor{teal!9}0.17 & \cellcolor{marron!60}32.45 & \cellcolor{marron!36}43.45 \\ 
         Qwen-2-Instruct  & \cellcolor{teal!9}0.78 & \cellcolor{teal!18}6.71 & \cellcolor{teal!15}70.16 \\ 
         LLaMA-3.1-Instruct  & \cellcolor{marron!3}4.80 & \cellcolor{teal!27}\textbf{3.76} & \cellcolor{marron!100}4.71 \\
         Qwen-2.5-Instruct  & \cellcolor{marron!3}4.78 & \cellcolor{teal!24}4.60 & \cellcolor{teal!27}74.08\\
         \midrule
         GPT-3.5-turbo & \cellcolor{teal!3}2.36 & \cellcolor{teal!6}10.73 & \cellcolor{teal!72}89.79 \\
         GPT-4o & \cellcolor{teal!9}\textbf{0.12} & \cellcolor{teal!24}4.64 & \cellcolor{teal!66}87.43 \\
         GPT-4-turbo & \cellcolor{teal!9}0.34 & \cellcolor{teal!15}7.83 & \cellcolor{teal!75}\textbf{90.31} \\ \midrule
         TL-CodeLLaMA-2 & \cellcolor{teal!9}0.82 & \cellcolor{teal!24}4.84 & \cellcolor{teal!36}77.75 \\ \bottomrule
    \end{tabular}
    }

\end{table}

\subsection{Metrics}

For \textbf{single-turn} tool use, where the original dataset provides a standard answer, we follow~\citet{RoTBench} and assess the model's performance across three key areas:  
    1) \textbf{Tool Selection (TS)} measures the model's accuracy in selecting the tool specified by the standard answer;
    2) \textbf{Parameter Identification (PI)} evaluates the model's ability to correctly select the tool and identify the relevant parameters required for invocation;
    and 3) \textbf{Content Filling (CF)} assesses the model's capacity to complete the single-turn tool invocation, including selecting the correct tool, identifying relevant parameters, and filling in the appropriate values.

For \textbf{multi-turn} tool use, where no standardized interaction path exists, we adapt the methods of \citet{Toolllm} and \citet{ToolEyes}, and assess performance based on following metrics:  
    1) \textbf{Documentation Understanding Error (DE)} represents the percentage of errors resulting from the model's failure to interpret the tool documentation, encompassing tool hallucinations, parameter hallucinations, and missing necessary parameters;
    2) \textbf{Tool Call Error (CE)} denotes the proportion of errors arising from incorrect tool invocation, covering all error types except those classified as DE;
    and 3) \textbf{Valid Answers (VA)} evaluates the percentage of instances where the model delivers valid responses within nine turns.

\begin{table*}[!t]
    \caption{The ablation studies of the three components of our method, with better results than \textbf{Standard SFT} labeled in \colorbox{teal}{teal}, poorer results labeled in \colorbox{marron}{purple}. Darker colors indicate larger gaps.}
    \label{tab:ablation}
    \centering
    \resizebox{\linewidth}{!}
    {
    \begin{tabular}{lccccccccc}
    \toprule
        \multirow{2}*{\textbf{Models}} & \multicolumn{3}{c}{\textbf{ToolAlpaca}} & \multicolumn{3}{c}{\textbf{RoTBench}} & \multicolumn{3}{c}{\textbf{ToolEyes}}  \\ \cmidrule(lr){2-4} \cmidrule(lr){5-7} \cmidrule(lr){8-10}
         & \textbf{TS} ($\uparrow$) & \textbf{PI} ($\uparrow$) & \textbf{CF} ($\uparrow$) & \textbf{TS} ($\uparrow$) & \textbf{PI} ($\uparrow$) & \textbf{CF} ($\uparrow$) & \textbf{DE} ($\downarrow$)  & \textbf{CE} ($\downarrow$) & \textbf{VA} ($\uparrow$)\\ \midrule
        Standard SFT (\textit{w/ None}) & \textit{74.56} & \textit{70.18 }& \textit{42.98 }& \textit{76.19} & \textit{55.24} & \textit{38.10} & \textit{0.72} & \textit{9.08} &\textit{ 59.32} \\ \midrule
         \textit{~~~~w/ MAE} & \cellcolor{marron!3}73.68 & \cellcolor{marron!15}64.91 & \cellcolor{teal!3}43.68 & \cellcolor{teal!6}78.10 & \cellcolor{teal!6}57.14 & \cellcolor{teal!6}40.95 & \cellcolor{marron!2}1.22 & \cellcolor{teal!12}5.53 & \cellcolor{teal!27}68.85 \\
         \textit{~~~~w/ PKT} & \cellcolor{teal!9}77.78 & \cellcolor{teal!3}71.72 & \cellcolor{teal!3}43.86 &\cellcolor{teal!18}82.86 & \cellcolor{teal!24}63.81 & \cellcolor{teal!30}48.57  & \cellcolor{marron!1}1.07 & \cellcolor{teal!6}7.52 & \cellcolor{teal!12}63.61  \\
         \textit{~~~~w/ IRM} & \cellcolor{teal!42}88.60 & \cellcolor{teal!42}84.21 & \cellcolor{teal!45}57.02 & \cellcolor{teal!9}79.05 & \cellcolor{teal!12}59.05 & \cellcolor{teal!18}44.76 & \cellcolor{teal!1}0.65 & \cellcolor{teal!9}6.80 & \cellcolor{marron!6}57.33 \\ \midrule
         \textit{~~~~w/ MAE \& PKT} & \cellcolor{marron!3}73.68 & \cellcolor{marron!15}64.91 & \cellcolor{teal!3}43.86 & \cellcolor{teal!12}80.95 & \cellcolor{teal!3}56.19 & \cellcolor{teal!6}40.00 & 0.71 & \cellcolor{marron!3}10.07 & \cellcolor{teal!48}75.13 \\ 
         \textit{~~~~w/ MAE \& IRM} & \cellcolor{teal!39}87.72 & \cellcolor{teal!24}78.95 & \cellcolor{teal!45}57.89 & \cellcolor{teal!18}82.86 & \cellcolor{teal!3}56.19 & \cellcolor{teal!24}45.71 & \cellcolor{marron!1}0.96 & \cellcolor{teal!12}5.03 & \cellcolor{teal!36}71.20 \\
         \textit{~~~~w/ PKT \& IRM} & \cellcolor{teal!36}86.84 & \cellcolor{teal!27}79.82 & \cellcolor{teal!45}57.02 & \cellcolor{teal!18}82.86 & \cellcolor{teal!24}63.81 & \cellcolor{teal!30}48.57 & \cellcolor{marron!1}0.88 & \cellcolor{teal!6}7.11 & \cellcolor{teal!18}65.18 \\ \midrule
         
         ~~~~\textit{w/ All} (TL-CodeLLaMA-2) & \cellcolor{teal!39}87.72 & \cellcolor{teal!24}78.07 & \cellcolor{teal!45}57.89 & \cellcolor{teal!21}83.81 & \cellcolor{marron!3}54.29 & \cellcolor{teal!12}42.86 & \cellcolor{marron!1}0.82 & \cellcolor{teal!15}4.84 & \cellcolor{teal!54}77.75 \\ 
           \bottomrule
    \end{tabular}
    }

\end{table*}

\subsection{Implementation Details}
In the \textbf{SFT} stage, we use 1,217 constructed data samples, applying both the MAE and PKT strategies. We employ the AdamW optimizer~\cite{Adamw} with cosine scheduling, setting the learning rate to 1e-6, a warmup rate of 0.01, and a batch size of 4, training for a total of 1 epoch. For the PKT strategy, \(w_{max}\) is set to 9. 
In the \textbf{RL} stage, we filter 1,194 entries from the constructed data and apply PPO with the reward function described in Section~\ref{sec:rl}. The actor learning rate is set to 2e-6, the critic learning rate to 1e-6, and the batch size to 8, training for a total of 3 epochs.
For \textbf{testing}, we use the official prompt template for tool invocation and apply greedy search.\footnote{Templates for each LLM are provided in Appendix~\ref{sec:template}.}

\section{Experiments}
\label{sec:experiments}

\subsection{Main Results}
\label{sec:main}

We evaluate the performance of various LLMs on three single-turn tool-use test sets and one multi-turn tool-use test set, with the results summarized in Table~\ref{tab:main_single} and Table~\ref{tab:main_multi}.\footnote{Since NexusRaven-2-13B does not receive tool feedback during interactions, it is evaluated only on the single-turn tool-use datasets.}
Despite using the smallest model size and the least amount of training data, our approach achieves results comparable to the best-performing models. These findings demonstrate the potential for smaller, efficient models to excel in the tool use task, making advanced capabilities more accessible for resource-constrained environments.

\paragraph{Single-Turn Evaluation}
Results from three single-turn tool-use test sets demonstrate that TL-CodeLLaMA-2, with only 7B parameters and 1,217 training examples, surpasses all other open-source LLMs in overall task completion (i.e., CF). Remarkably, on the ToolAlpaca and BFCL-v3 datasets, TL-CodeLLaMA-2 outperforms GPT-4-turbo, the top-performing GPT family model, by an impressive 15.78\% and 12.08\%, respectively. Most notably, TL-CodeLLaMA-2 is the \emph{only} model
that consistently exceeds the average performance across all three aspects of every dataset evaluated, highlighting the effectiveness of our approach in enhancing single-turn tool usage capabilities.

\paragraph{Multi-Turn Evaluation}
In the multi-turn test set, our approach significantly enhances the ability of LLMs to handle the tool use task. TL-CodeLLaMA-2 achieves an total error rate of just 5.64\% on the test set, second only to GPT-4o, which had the lowest error rate at 4.76\%, and ahead of Qwen-2-Instruct at 7.49\%. Additionally, TL-CodeLLaMA-2 maintaines a low error rate while achieving a high effective response rate, outperforming all other open-source models. In contrast, LLaMA-3.1-Instruct-8B frequently fails to provide a valid direct answer, effectively rendering it incapable of completing the task. These results highlight that our trained model effectively uses tools in multi-turn settings to solve complex user queries.

\subsection{Ablation Studies}
\label{sec:ablation}

To assess the individual contributions of the three components in our design that enhance LLMs' tool-use capabilities, we conduct ablation studies, comparing model performance across various scenarios. The results are shown in Table~\ref{tab:ablation}.

When compared to standard SFT without additional techniques (i.e., w/ None), masking erroneous interaction paths during training (e.g., w/ MAE) reduces the model's overall error rate in multi-turn tool use by nearly one-third and increases effective responses by 9.53\%. This suggests that removing these erroneous paths prevents LLMs from learning incorrect tool-use patterns. However, since such errors are rare in single-turn tool use, the improvement in that case is less pronounced and requires additional techniques for further gains (e.g., w/ MAE \& IRM). Similarly, optimizing the weights of key tokens during training (e.g., w/ PKT) enhances the model’s ability to differentiate between similar tools, improving tool selection accuracy and overall performance. Furthermore, reinforcement learning with our proposed reward function (e.g., w/ IRM) further improves performance across all stages by dynamically optimizing the entire tool-use process, addressing diverse errors encountered during tool use. Finally, TL-CodeLLaMA-2 (i.e., w/ All), which integrates all three strategies, maximizes their combined advantages, significantly improving LLM performance in both single-turn and multi-turn tool use with only 1,217 data points, demonstrating the effectiveness of our approach.

\begin{figure}[!t]
    \centering
    \includegraphics[width=\linewidth]{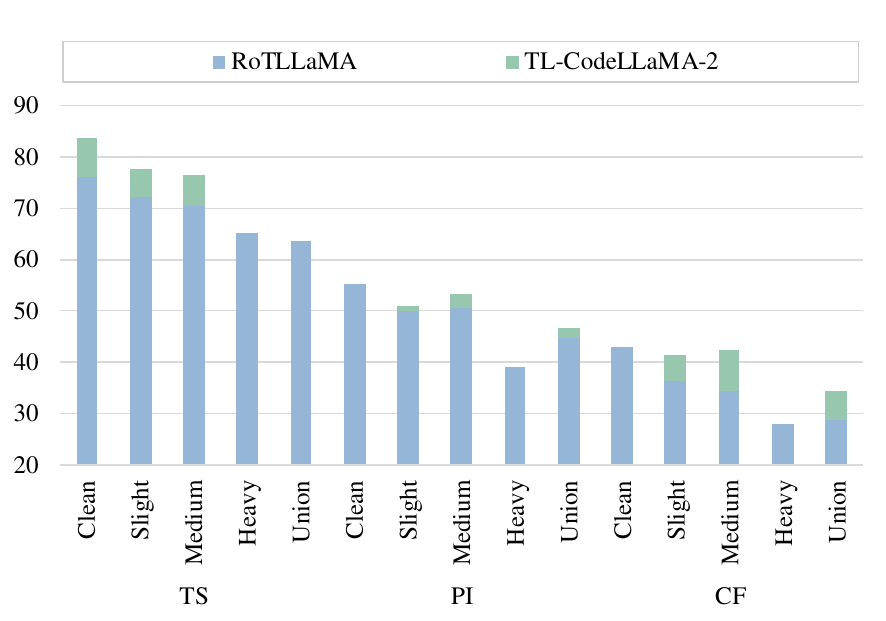}
    \caption{Performance comparison of RoTLLaMA and TL-CodeLLaMA-2 in different noise environments. RoTLLaMA's results are from~\citet{RoTBench}.}
    \label{fig:robust}

\end{figure}

\section{Further Studies}

\subsection{Robustness Improvement}
\label{sec:robustness}

In real-world environments, tools often contain various types of noise, and LLMs must be robust in their tool use to effectively meet user needs across different situations. RoTBench provides five tool-use test environments with varying noise levels, designed to evaluate whether LLMs can accurately understand the functions and properties of different tools and execute effective invocations. We compare the performance of TL-CodeLLaMA-2 and RoTLLaMA across these five noisy environments, as shown in Figure~\ref{fig:robust}. While RoTLLaMA has been optimized for such environments through targeted noise augmentation, TL-CodeLLaMA-2, without specific optimizations, matches or exceeds RoTLLaMA's performance in all aspects. This suggests that our approach allows the model to focus on the core functionality of external tools without being hindered by noise, making it more adaptable to real-world scenarios.

\begin{figure}[!t]
    \centering
    \includegraphics[width=\linewidth]{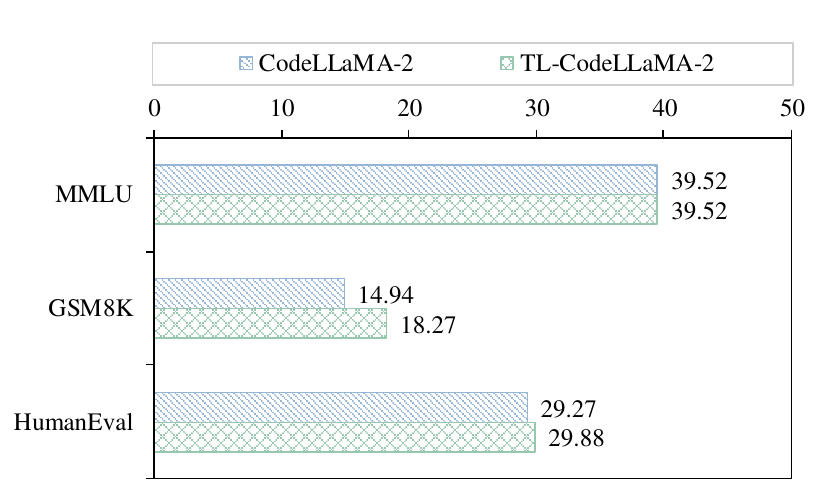}
    \caption{Performance comparison of CodeLLaMA-2 and TL-CodeLLaMA-2 across various general tasks.}
    \label{fig:general}

\end{figure}

\subsection{General Performance}
\label{sec:general}

The strong performance of LLMs is largely attributed to their extensive world knowledge and generalizability acquired during pre-training~\cite{analy-ye}. However, fine-tuning on domain-specific tasks can sometimes compromise this generalizability~\cite{SFT-1, SFT-2}. To assess whether TL-CodeLLaMA-2's general-purpose capabilities are affected by its exclusive training on tool-use data, we evaluate its performance on three general test sets: MMLU (knowledge)~\cite{MMLU}, GSM8K (math)~\cite{GSM8K}, and HumanEval (code)~\cite{HumanEval}, comparing it to CodeLLaMA-2-7B. As shown in Figure~\ref{fig:general}, despite being fine-tuned solely for the tool-use task, TL-CodeLLaMA-2 retains its original task performance and even shows slight improvements in math and coding abilities. This is because our method requires only a small amount of training data, resulting in minimal changes to the model’s original parameters, thus preserving its knowledge base~\cite{IFT, knowledge, QA}. Furthermore, the enhancement in tool-use ability appears to improve the model's reasoning capacity, contributing to better performance in math and coding tasks. These findings further underscore the broad applicability of our approach.

\section{Conclusion}
In this paper, we introduce TL-Training, a novel paradigm for training LLMs specifically for tool use. Our approach mitigates the impact of erroneous interaction data, adaptively adjusts token weights, and introduces a reward mechanism tailed for tool use to facilitate PPO-based reinforcement learning. This methodology not only enhances LLMs' tool-use capabilities but also improves their robustness in noisy environments, all while preserving strong general performance across a range of tasks. Our findings demonstrate the effectiveness of TL-Training in addressing real-world challenges in tool use, offering a promising direction for future research in improving LLM interaction capabilities and adaptability.

\section*{Limitations}

While we propose a novel paradigm for training LLMs in tool use, our work still has a few limitations. First, we do not construct large-scale training data. However, despite using only 1,217 data samples, our results show that we match or even surpass the best current tool-use performance, highlighting the strengths of our approach. Second, we design a reward function based on tool feedback for tool use directly, without training a separate reward model. Nonetheless, we experimentally demonstrate the effectiveness of our reward function. In future work, we plan to explore training a reward model specifically for tool learning to further improve model performance.

\section*{Acknowledgments}
The authors wish to thank the anonymous reviewers for their helpful comments. This work was partially funded by National Natural Science Foundation of China (No.62476061,62206057), Shanghai Rising-Star Program (23QA1400200), Natural Science Foundation of Shanghai (23ZR1403500).

\bibliography{custom}

\newpage
\appendix
\onecolumn
\section{Theorems and Proofs}
\label{sec:proof}
In this paper, we propose $\mathcal{L}_{MAE}$ and $\mathcal{L}_{PKT}$, aimed at enhancing the model's ability to utilize tools during the SFT stage. Additionally, we provide theoretical explanations of the effectiveness of these strategies.

\begin{theorem}
During the SFT stage for LLM in tool use, gradient updates resulting from incorrect interaction paths in the training data can adversely impact the model's ability to choose the appropriate tool.
\end{theorem}

\begin{proof}
Let \(\mathcal{M}\) be an LLM that interacts with a set of tools \(\mathbb{T}\) to answer the user query \(q\). At each step \(s\), the model selects a tool \(t_s \in \mathbb{T}\) based on the query and the history of tool selections and feedback:
\[
t_{s+1} = \mathcal{M}(\cdot \mid q, \mathbb{T}, t_{0..s}, o_{0..s}),
\]
where \(o_s\) is the feedback received after calling tool \(t_s\).

Consider a dataset \(\mathbb{D}\) comprising interaction sequences \((q, t_{0..s}, o_{0..s})\), which includes both correct and erroneous tool calls. Let \(\mathbb{T}_e \subseteq \{t_0, t_1, \dots, t_s\}\) denote the set of erroneous tool calls identified via analysis of feedback \(o_s\).

The standard loss function during SFT aims to maximize the likelihood of the model's tool selections over the entire dataset:
\[
\mathcal{L} = -\sum_{\mathbb{D}} \sum_{t_s} \log p_\mathcal{M}(t_{s} \mid q, \mathbb{T}, t_{0..s-1}, o_{0..s-1}).
\]

The gradient of this loss with respect to the model parameters \(\theta\) is:
\[
\nabla_\theta \mathcal{L} = -\sum_{\mathbb{D}} \sum_{t_s} \nabla_\theta \log p_\mathcal{M}(t_{s} \mid q, \mathbb{T}, t_{0..s-1}, o_{0..s-1}).
\]

This gradient comprises contributions from both correct and erroneous tool calls. The gradient component arising from erroneous tool calls is:
\[
G_{\text{error}} = -\sum_{\mathbb{D}} \sum_{t_{s} \in \mathbb{T}_e} \nabla_\theta \log p_\mathcal{M}(t_{s} \mid q, \mathbb{T}, t_{0..s-1}, o_{0..s-1}).
\]

These gradients encourage the model to replicate erroneous tool selections, thereby misguiding its learning process. Specifically, they can increase the likelihood of the model making incorrect tool calls in future interactions, which negatively impacts its performance.
By including erroneous tool calls in the gradient updates, the model parameters \(\theta\) are adjusted in directions that do not align with optimal decision-making. This is detrimental because it interferes with the model's ability to learn the correct sequence of tool selections that effectively resolve user queries.
Therefore, errors in the training data introduce gradient updates that adversely affect the model's performance.
\end{proof}

To mitigate this effect, we propose to modify the loss function to exclude erroneous tool calls from back-propagation:
\[
\mathcal{L}_{{MAE}} = -\sum_{\mathbb{D}} \sum_{t_{s} \notin \mathbb{T}_e} \log p_\mathcal{M}(t_{s} \mid q, \mathbb{T}, t_{0..s-1}, o_{0..s-1}).
\]
By omitting the erroneous tool calls from the loss computation, their associated gradients are not used to update the model parameters. This reduces the harmful impact of errors in the training data on the model's performance.

\begin{theorem}
During the gradient update process in SFT, assigning higher weights to key tokens prioritizes their contribution to the loss function, enabling the model to focus more on these tokens and fit them better.
\end{theorem}
\begin{proof}
Let \(\mathcal{M}\) be an LLM interacting with a set of tools \(\mathbb{T}\) to answer the user query \(q\). Each tool \(t_s\) used at step \(s\) is a sequence of tokens:
\[
t_s = (t_s^0, t_s^1, \ldots, t_s^{l_s}),
\]
where \(l_s\) is the length of the token sequence for tool \(t_s\).

For each tool \(t_s\), we categorize its tokens into two sets:
\[
K_s = \{ t_s^m \in t_s \mid t_s^m \text{ is a key token} \},
\]
\[
NK_s = \{ t_s^m \in t_s \mid t_s^m \text{ is not a key token} \}.
\]

We assign weights \(w_s^m\) to each token \(t_s^m\) as follows:
\[
w_s^m =
\begin{cases}
\text{CLIP}\left( \frac{|NK_s|}{|K_s|},\ 1,\ w_{\text{max}} \right) & \text{if } t_s^m \in K_s, \\
1 & \text{if } t_s^m \in NK_s,
\end{cases}
\]
where \(|K_s|\) and \(|NK_s|\) denote the number of key and non-key tokens in \(t_s\), respectively, \(w_{\text{max}}\) is the maximum adjustment multiplier, and \(\text{CLIP}(x,\ \text{min},\ \text{max})\) constrains \(x\) to the interval \([\text{min},\ \text{max}]\).

The modified loss function that prioritizes key tokens is:
\[
\mathcal{L}_{PKT} = -\sum_{\mathbb{D}} \sum_{t_s} \sum_{t_s^m} w_s^m \cdot \log p_\mathcal{M}\left( t_s^m \mid q,\ \mathbb{T},\ t_{0..s-1},\ o_{0..s-1},\ t_s^0 \ldots t_s^{m-1} \right),
\]
where \(p_\mathcal{M}\) is the probability assigned by the model to token \(t_s^m\), given the context.

The gradient of the loss with respect to the model parameters \(\theta\) is:
\[
\nabla_\theta \mathcal{L}_{PKT} = -\sum_{\mathbb{D}} \sum_{t_s} \sum_{t_s^m} w_s^m \cdot \nabla_\theta \log p_\mathcal{M}\left( t_s^m \mid q,\ \mathbb{T},\ t_{0..s-1},\ o_{0..s-1},\ t_s^0 \ldots t_s^{m-1} \right).
\]

Tokens with higher weights \(w_s^m\) contribute more to the gradient:
\[
\| w_s^m \cdot \nabla_\theta \log p_\mathcal{M}\left( t_s^m \mid q,\ \mathbb{T},\ t_{0..s-1},\ o_{0..s-1},\ t_s^0 \ldots t_s^{m-1} \right) \| \propto w_s^m.
\]

Assuming \(w_s^k > w_s^n\) for key token \(t_s^k\) and non-key token \(t_s^n\), the gradient contribution from \(t_s^k\) is larger:
\[
\| w_s^k \cdot \nabla_\theta \log p_\mathcal{M}\left( t_s^k \mid \cdots \right) \| > \| w_s^n \cdot \nabla_\theta \log p_\mathcal{M}\left( t_s^n \mid \cdots \right) \|.
\]

During gradient descent, the parameter updates prioritize reducing the loss associated with higher-weighted (key) tokens:
\[
\theta' = \theta - \eta \nabla_\theta \mathcal{L}_{PKT}.
\]

As a result, the model adjusts its parameters more significantly to fit the key tokens, improving its ability to generate them correctly. This leads to better fitting of tokens with higher weights.
Therefore, assigning higher weights to key tokens during gradient updates enhances the model's performance on these tokens.
\end{proof}

\clearpage

\section{Details of Datasets}
\label{sec:detail-datasets}

As shown in Table~\ref{tab:dataset}, we construct a custom training set focused on multi-turn tool use and evaluate it using four publicly available test sets.

\paragraph{Training Data}
To train the LLMs using our method, we first construct a training dataset. Since ToolEyes provides a comprehensive set of invocable tools, we use it as a foundation to artificially create 1,217 relevant user requirements. GPT-4o is then employed to interact with these tools and generate the corresponding tool usage trajectories, which form our training set.

While previous studies often construct over 100,000 data points for training~\cite{Toolllm}, we deliberately limit our dataset size. Our main goal is to validate the effectiveness of our approach rather than to scale data volume. Surprisingly, the experimental results in Section~\ref{sec:experiments} show that training on just 1,217 data points using our method matches or even exceeds the performance of leading LLMs.

\paragraph{Test Sets}
To comprehensively evaluate LLM tool-use performance, we use four open-source tool usage test sets. ToolAlpaca (eval\_real), RoTBench (Clean), and BFCL-v3 (executable) are selected for single-turn tool use evaluations, while ToolEyes is used for multi-turn tool use assessment.

\clearpage

\section{Details of Baselines}
\label{sec:detail-base}
In this paper, we select nine existing LLMs from three different sources for a comprehensive comparison with our tool-use LLMs.

\paragraph{Tool-Use LLMs}
\textbf{ToolLLaMA-2-7B} and \textbf{NexusRaven-2-13B} are prominent tool-use LLMs, built on LLaMA-2-7B and CodeLLaMA-2-13B, respectively. These models are trained on a large volume of tool-use data, enhancing their ability to interact with tools. For example, ToolLLaMA-2-7B is trained on over 120,000 data points covering more than 16,000 tools using standard SFT, significantly boosting its tool-use capabilities.

\paragraph{Open-Source LLMs}
Among the existing open-source, general-purpose LLMs, \textbf{ChatGLM-4-chat-9B}, \textbf{Qwen-2-Instruct-7B}, \textbf{LLaMA-3.1-Instruct-8B} and \textbf{Qwen-2.5-Instruct-7B} have been specifically optimized for tool use, enabling them to interact with various tools to fulfill user needs.

\paragraph{Closed-Source LLMs}
The GPT family represents some of the most advanced LLMs, demonstrating strong performance not only in general-purpose tasks but also in tool use, with notable generalization capabilities. For this study, we select \textbf{GPT-3.5-turbo}, \textbf{GPT-4o}, and \textbf{GPT-4-turbo} as leading representatives of the GPT series for comparison.

\paragraph{Our Model}
We apply the TL-Training paradigm to CodeLLaMA-2-7B, using the custom dataset of 1,217 examples, to develop \textbf{TL-CodeLLaMA-2}, a specialized tool-use LLM. Compared to the other models in this study, TL-CodeLLaMA-2 is the smallest and trained on the least amount of data.

\clearpage

\section{Prompt Template}
\label{sec:template}
We use the official prompt of each LLM for tool use, which are provided from Table~\ref{tab:template-toolllama} to Table~\ref{tab:template-tl}.

\begin{table}[H]
    \caption{An example for tool use of ToolLLaMA-2-7B.}
    \label{tab:template-toolllama}
    \centering
    \begin{tabular}{p{0.95\linewidth}}
    \toprule
    System:\newline
    You are AutoGPT, you can use many tools(functions) to do the following task.\newline
    First I will give you the task description, and your task start.\newline
    At each step, you need to give your thought to analyze the status now and what to do next, with a function call to actually excute your step. Your output should follow this format:\newline
    Thought:\newline
    Action\newline
    Action Input:\newline
    \newline
    After the call, you will get the call result, and you are now in a new state.\newline
    Then you will analyze your status now, then decide what to do next...\newline
    After many (Thought-call) pairs, you finally perform the task, then you can give your finial answer.\newline
    Remember: \newline
    1.the state change is irreversible, you can't go back to one of the former state, if you want to restart the task, say \"I give up and restart\".\newline
    2.All the thought is short, at most in 5 sentence.\newline
    3.You can do more then one trys, so if your plan is to continusly try some conditions, you can do one of the conditions per try.\newline
    Let's Begin!\newline
    Task description: You should use functions to help handle the real time user querys. Remember:\newline
    1.ALWAYS call \"Finish\" function at the end of the task. And the final answer should contain enough information to show to the user,If you can't handle the task, or you find that function calls always fail(the function is not valid now), use function Finish->give\_up\_and\_restart.\newline
    2.Do not use origin tool names, use only subfunctions' names.\newline
    \newline
    Specifically, you have access to the following APIs: [\{``type": ``function", ``function": \{``name": ``random\_advice", ``description": ``Returns a random advice slip as a slip object.", ``parameters": \{``type": ``object", ``properties": \{\}, ``required": []\}\}\}]\newline
    User: Can you fetch some random advice for me?\newline
    Assistant:\\ \bottomrule
    \end{tabular}
\end{table}

\begin{table}[H]
    \caption{An example for tool use of NexusRaven-2-13B.}
    \label{tab:template-nexusraven}
    \centering
    \begin{tabular}{p{0.95\linewidth}}
    \toprule
    Function:\newline
    def random\_advice():\newline
    $~~~~$``````\newline
    $~~~~$Returns a random advice slip as a slip object.\newline
    \newline
    $~~~~$Args: \newline
    \newline
    $~~~~$Returns:\newline
    $~~~~$string: The feedback from the tool.\newline
    $~~~~$''''''\newline
    \newline
    User Query: Can you fetch some random advice for me?$<$human\_end$>$
    \\ \bottomrule
    \end{tabular}
\end{table}

\clearpage

\begin{CJK}{UTF8}{gbsn}
\begin{table}[H]
    \caption{An example for tool use of ChatGLM-4-chat-9B.}
    \label{tab:template-chatglm}
    \centering
    \begin{tabular}{p{0.95\linewidth}}
    \toprule
    $<|$system$|>$\newline
    你是一个名为 ChatGLM 的人工智能助手。你是基于智谱AI训练的语言模型 GLM-4 模型开发的，你的任务是针对用户的问题和要求提供适当的答复和支持。\newline
    \newline
    \# 可用工具\newline
    \newline
    \#\# random\_advice\newline
    \newline
    \{``name": ``random\_advice", ``description": ``Returns a random advice slip as a slip object.", ``parameters": \{``type": ``object", ``properties": \{\}, ``required": []\}\}\newline
    在调用上述函数时，请使用 Json 格式表示调用的参数。$<|$user$|>$\newline
    Can you fetch some random advice for me?$<|$assistant$|>$
    \\ \bottomrule
    \end{tabular}
\end{table}
\end{CJK}

\begin{table}[H]
    \caption{An example for tool use of Qwen-2-Instruct-7B.}
    \label{tab:template-qwen}
    \centering
    \begin{tabular}{p{0.95\linewidth}}
    \toprule
    $<|$im\_start$|>$system\newline
    You are a helpful assistant.$<|$im\_end$|>$$<|$im\_start$|>$user\newline
    Answer the following questions as best you can. You have access to the following tools:\newline
    \newline
    random\_advice: Call this tool to interact with the random\_advice API. What is the random\_advice API useful for? Returns a random advice slip as a slip object. Parameters: [] Format the arguments as a JSON object.\newline
    \newline
    Use the following format:\newline
    \newline
    Question: the input question you must answer\newline
    Thought: you should always think about what to do\newline
    Action: the action to take, should be one of [random\_advice]\newline
    Action Input: the input to the action\newline
    Observation: the result of the action\newline
    ... (this Thought/Action/Action Input/Observation can be repeated zero or more times)\newline
    Thought: I now know the final answer\newline
    Final Answer: the final answer to the original input question\newline
    \newline
    Begin!\newline
    \newline
    Question: Can you fetch some random advice for me?$<|$im\_start$|>$assistant
    \\ \bottomrule
    \end{tabular}
\end{table}

\clearpage

\begin{table}[H]
    \caption{An example for tool use of LLaMA-3.1-Instruct-8B.}
    \label{tab:template-llama3.1}
    \centering
    \begin{tabular}{p{0.95\linewidth}}
    \toprule
    $<|$begin\_of\_text$|><|$start\_header\_id$|>$system$<|$end\_header\_id$|>$\newline\newline
    Environment: ipython\newline
    Cutting Knowledge Date: December 2023\newline
    Today Date: 26 Jul 2024\newline \newline
    $<|$eot\_id$|><|$start\_header\_id$|>$user$<|$end\_header\_id$|>$\newline\newline
    Given the following functions, please respond with a JSON for a function call with its proper arguments that best answers the given prompt.\newline\newline
    Respond in the format {``name'': function name, ``parameters'': dictionary of argument name and its value}.Do not use variables.\newline\newline
    \{\newline
    ``function'': \{\newline
    ``description'': ``Returns a random advice slip as a slip object.'',\newline
    ``name'': ``random\_advice'',\newline
    ``parameters'': \{\newline
    ``properties'': \{\},\newline
    ``required'': [],\newline
    ``type'': ``object''\newline
    \}\newline
    \},\newline
    ``type'': ``function''\newline
    \}\newline\newline
    Can you fetch some random advice for me?$<|$eot\_id$|><|$start\_header\_id$|>$assistant$<|$end\_header\_id$|>$\newline
    \\ \bottomrule
    \end{tabular}
\end{table}

\begin{table}[H]
    \caption{An example for tool use of Qwen-2.5-Instruct-7B.}
    \label{tab:template-qwen2.5}
    \centering
    \begin{tabular}{p{0.95\linewidth}}
    \toprule
    $<|$im\_start$|>$system\newline
    You are Qwen, created by Alibaba Cloud. You are a helpful assistant.\newline\newline
    \# Tools\newline\newline
    You may call one or more functions to assist with the user query.\newline\newline
    You are provided with function signatures within $<$tools$>$$<$/tools$>$ XML tags:\newline
    $<$tools$>$\newline
    \{``function'': \{``description'': ``Returns a random advice slip as a slip object.'', ``name'': ``random\_advice'', ``parameters'': \{``properties'': \{\}, ``required'': [], ``type'': ``object''\}\}, ``type'': ``function''\}\newline
    $<$/tools$>$\newline\newline
    For each function call, return a json object with function name and arguments within $<$tool\_call$>$$<$/tool\_call$>$ XML tags:\newline
    $<$tool\_call$>$\newline
    \{``name'': $<$function-name$>$, ``arguments'': $<$args-json-object$>$\}\newline
    $<$/tool\_call$>$$<|$im\_end$|>$\newline
    $<|$im\_start$|>$user\newline
    Can you fetch some random advice for me?$<|$im\_end$|>$\newline
    $<|$im\_start$|>$assistant
    \\ \bottomrule
    \end{tabular}
\end{table}

\begin{table}[H]
    \caption{An example for tool use of GPT series models. The tools are send with the `tools' key of `OpenAI().chat.completions.create()'.}
    \label{tab:template-gpt}
    \centering
    \begin{tabular}{p{0.95\linewidth}}
    \toprule
    Can you fetch some random advice for me?
    \\ \bottomrule
    \end{tabular}
\end{table}

\begin{table}[H]
    \caption{An example for tool use of TL-CodeLLaMA-2.}
    \label{tab:template-tl}
    \centering
    \begin{tabular}{p{0.95\linewidth}}
    \toprule
    System: Function:\newline
    def random\_advice():\newline
    $~~~~$``````\newline
    $~~~~$Returns a random advice slip as a slip object.\newline
    \newline
    $~~~~$Args: \newline
    $~~~~$''''''\newline
    \newline
    User: Can you fetch some random advice for me?\newline
    Assistant:\\ \bottomrule
    \end{tabular}
\end{table}

\clearpage

\section{Examples of Tool Call and Feedback}
\label{sec:information}

In Table~\ref{tab:example}, we present some examples of potential scenarios that may arise during tool calls.

\begin{table}[H]
    \centering
    \caption{Examples of various scenarios that may arise during tool calls.}
    \label{tab:example}
    \resizebox{\linewidth}{!}
    {
    \begin{tabular}{m{0.17\linewidth} m{0.16\linewidth} m{0.27\linewidth} m{0.3\linewidth}}
    \toprule
     \textbf{Category}    & \textbf{Models}  & \textbf{Output}  & \textbf{Feedback} \\ \midrule
     All Right & TL-CodeLLaMA-2 & random\_advice() & \{``slip'': \{``id'': 52, ``advice'': ``Don't promise what you can't deliver.''\}\}\\
      Tool Instability   & GPT-4-turbo & \{``arguments'': ``\{\}'', ``name'': ``get\_currency''\} & HTTPSConnectionPool(host=..., port=443): Max retries exceeded with url: ...\\
      
      Tool Instability   &ChatGLM-4-chat &get\_exchange\_rate\newline
      \{``from\_currency'': ``EUR'', ``to\_currency'': ``GBP''\} & Invalid API call. Currency codes might be invalid.\\
      
      Tool Instability   & Qwen-2-Instruct & Thought: ...\newline
      Action: google\_trends\_search\newline
      Action Input: \{``query'': ``school'', ``data\_type'': ``related\_topics''\} & 400 Client Error: Bad Request for url: ...\\
      
      Tool Call Failure   &ChatGLM-4-chat & search\_country\newline
      \{``query'': ``IN''\}& \{``error'': ``Response error.''\}\\
      
      Tool Call Failure   & GPT-3.5-turbo & \{``arguments'': ``\{``query'': ``Bitcoin'', ``data\_type'': ``RELATED\_QUERIES'', ``limit'': ``5''\}'', ``name'': ``google\_trends\_search''\} & `$<=$' not supported between instances of `str' and `int'\\
      
      Tool Call Failure   & ToolLLaMA-2 & Thought: ...\newline
      Action: get\_nobel\_results\newline
      Action Input: \{
      ``year'': ``2018''
        \} & Object of type bytes is not JSON serializable\\
      
      Tool Hallucination   &Qwen-2.5-Instruct & \{``arguments'': ``\{``query'': ``most popular tourist destinations in Europe''\}'', ``name'': ``google\_search''\} & name google\_search if not defined\\
      
     Parameter Hallucination   & ToolLLaMA-2 & Thought: ...\newline
     Action: get\_news\_headlines\newline
     Action Input: \{  ``api\_key'': ``your\_api\_key'',
     ``q'': ``technology'',
     ``sortBy'': ``popularity''\} & get\_news\_headlines() got an unexpected keyword argument `sortBy'\\
     
     Parameter Missing   &Qwen-2.5-Instruct & \{``arguments'': ``\{``q'': ``London'', ``days'': ``7'', ``api\_key'': ``your\_api\_key''\}'', ``name'': ``forecast''\} & forecast() missing 1 required positional argument: aqi \\
     
     Others   &GPT-4o &\{``arguments'': ``\{``website'': ``www.mywebsite.com''\}'', ``name'': ``analyze\_scan''\} & \{``error'': ``Recently completed scan for www.mywebsite.com not found''\} \\ \bottomrule
         
    \end{tabular}
    }
\end{table}

\end{document}